\documentclass[letterpaper, 10 pt, conference]{ieeeconf}  % Comment this line out
\IEEEoverridecommandlockouts                              % This command is only
\usepackage{url}
\usepackage{cite}
\usepackage{hyperref}
\usepackage{amsfonts}
\usepackage{amssymb}
\usepackage{mathtools}
\usepackage{acronym}
\usepackage{xcolor}
\usepackage{cases}
\usepackage[ruled,vlined]{algorithm2e}
\usepackage{booktabs}
\usepackage{multirow}
\newcommand{\rev}[1]{\textcolor{black}{#1}}

\usepackage{amsthm}

\acrodef{SVM}[SVM]{Support Vector Machine}
\acrodef{NEF}[NEF]{Neural Engineering Framework}
\acrodef{RC}[RC]{Reservoir Computing}
\acrodef{ESN}[ESN]{Echo State Network}
\acrodef{ELM}[ELM]{Extreme Learning Machine}
\acrodef{EDMD}[EDMD]{Extended Dynamic Mode Decomposition}
\acrodef{HAVOK}[HAVOK]{Hankel Alternative View of Koopman}
\acrodef{ESP}[ESP]{Echo State Property}
\acrodef{NRMSE}[NRMSE]{Normalized Root Mean Square Error}
\acrodef{RBF}[RBF]{Radial Basis Function}
\acrodef{HDMD}[HDMD]{Hankel Dynamic Mode Decomposition}

\newtheorem{assumption}{Assumption}
\newtheorem{theorem}{Theorem}[section]

\newtheorem{proposition}[theorem]{Proposition}
\theoremstyle{definition}

\theoremstyle{remark}
\newtheorem*{remark}{Remark}
\title{\LARGE \bf
Koopman Identification of Nonlinear Systems via Reservoir Liftings
}
\author{Weibin Gu$^{1}$, Chen Yang$^{1,2}$, Lu Shi$^{1}$ % <-this % stops a space
\thanks{$^{1}$The authors are with the Institute for AI Industry Research (AIR), Tsinghua University, Beijing, PR China
        {\tt\small \{guweibin, shilu\}@air.tsinghua.edu.cn}}
\thanks{$^{2}$Chen Yang is also with the School of Engineering, China University of Petroleum-Beijing at Karamay, Xinjiang Uygur Autonomous Region, PR China
        {\tt\small 2025216932@student.cup.edu.cn}}%
\thanks{This work was partly supported by the China Postdoctoral Science Foundation under Grant Number 2025M781650. }
}

\begin{document}

\maketitle
\thispagestyle{empty}
\pagestyle{empty}

%%%%%%%%%%%%%%%%%%%%%%%%%%%%%%%%%%%%%%%%%%%%%%%%%%%%%%%%%%%%%%%%%%%%%%%%%%%%%%%%
\begin{abstract}
Learning tractable linear representations of nonlinear dynamical systems via Koopman operator theory is often hindered by \rev{dictionary} selection, temporal memory encoding, and numerical ill-conditioning. Inspired by Reservoir Computing (RC) paradigm, this paper introduces the RC–Koopman framework, which interprets reservoir as a stateful, finite-dimensional Koopman dictionary whose temporal depth is explicitly controlled by its spectral radius. We show that the Echo State Property (ESP) guarantees well-posedness and favorable numerical conditioning of the lifted Koopman approximation. A correlation-based spectral radius selection algorithm aligns reservoir memory with dominant system timescales. Analysis reveals how the finite memory of the reservoir determines which Koopman eigenfunctions remain observable from the lifted features. \rev{Evaluation on synthetic benchmarks demonstrates that RC--Koopman achieves a favorable balance between reconstruction accuracy of the underlying nonlinear dynamics and dynamical stability, compared to} Extended Dynamic Mode Decomposition (EDMD) and Hankel-based lifting approaches. Code available at: \url{https://github.com/NEAR-the-future/RC-Koopman.git}
\end{abstract}

%%%%%%%%%%%%%%%%%%%%%%%%%%%%%%%%%%%%%%%%%%%%%%%%%%%%%%%%%%%%%%%%%%%%%%%%%%%%%%%%
\section{Introduction}
Learning tractable models of nonlinear dynamical systems remains a fundamental challenge in system identification and control. While nonlinear models can capture complex behaviors with high fidelity, they are often difficult to integrate with classical analysis, control synthesis, and optimization frameworks. A widely adopted strategy to address this difficulty is to \emph{lift} system observations into a higher-dimensional feature space where the dynamics admit an approximately linear representation. Typically, Koopman operator theory provides a principled foundation for this idea by showing that nonlinear state evolution can be represented as linear dynamics in an infinite-dimensional space of observables~\cite{koopman1931hamiltonian,brunton2022modern,shi2026koopman}. In practice, finite-dimensional approximations of this operator are typically obtained through data-driven identification. Prominent examples include \ac{EDMD}~\cite{williams2015data}, which employs predefined observable dictionaries, and delay-coordinate approaches such as \ac{HAVOK}~\cite{brunton2017chaos} and related Hankel-based methods~\cite{kamb2020time,hirsh2021structured}. Despite their success, these approaches face two fundamental challenges. First, dictionary-based methods rely on hand-crafted, memoryless observables, making the discovery of a Koopman-invariant subspace highly heuristic. Second, delay-coordinate embeddings incorporate temporal information but often suffer from dimensionality growth and severe numerical ill-conditioning due to the strong correlations present in Hankel matrices.

\ac{RC} architectures, such as \acp{ESN}, offer an alternative yet appealing paradigm for constructing expressive lifting representations by projecting inputs through a fixed, randomly initialized recurrent network~\cite{lukovsevivcius2009reservoir}. 
A driven reservoir establishes generalized synchronization with the input system, where the synchronization map typically forms a topological embedding; this ensures the reservoir state space preserves the geometric structure of the input attractor~\cite{hart2025generic}. Such representation allows approximation from the current reservoir state via a learned, usually linear map, which in Koopman theory corresponds to identifying the Koopman operator on an observable space.
\rev{Notably, we distinguish between \textit{stateless} liftings such as memoryless basis functions in \ac{EDMD} and \textit{stateful} liftings. 
As a stateful system, the reservoir possesses internal dynamics that enable temporal memory encoding, rendering the current lifted observable a functional of the entire input history. 
The stability of this representation is ensured by contractive dynamics\rev{~\cite{kaplan2026dominant}} under the \ac{ESP}, which guarantees the reservoir state remains a stable functional of past inputs~\cite{jaeger2001echo}. 
The temporal depth of this memory is explicitly modulated by the reservoir's spectral radius. }
Together, these properties establish reservoir dynamics as a natural stateful lifting mechanism for Koopman identification, capturing essential temporal dependencies without the need for explicit delay embeddings.

\begin{figure}
    \centering
    \includegraphics[width=\linewidth]{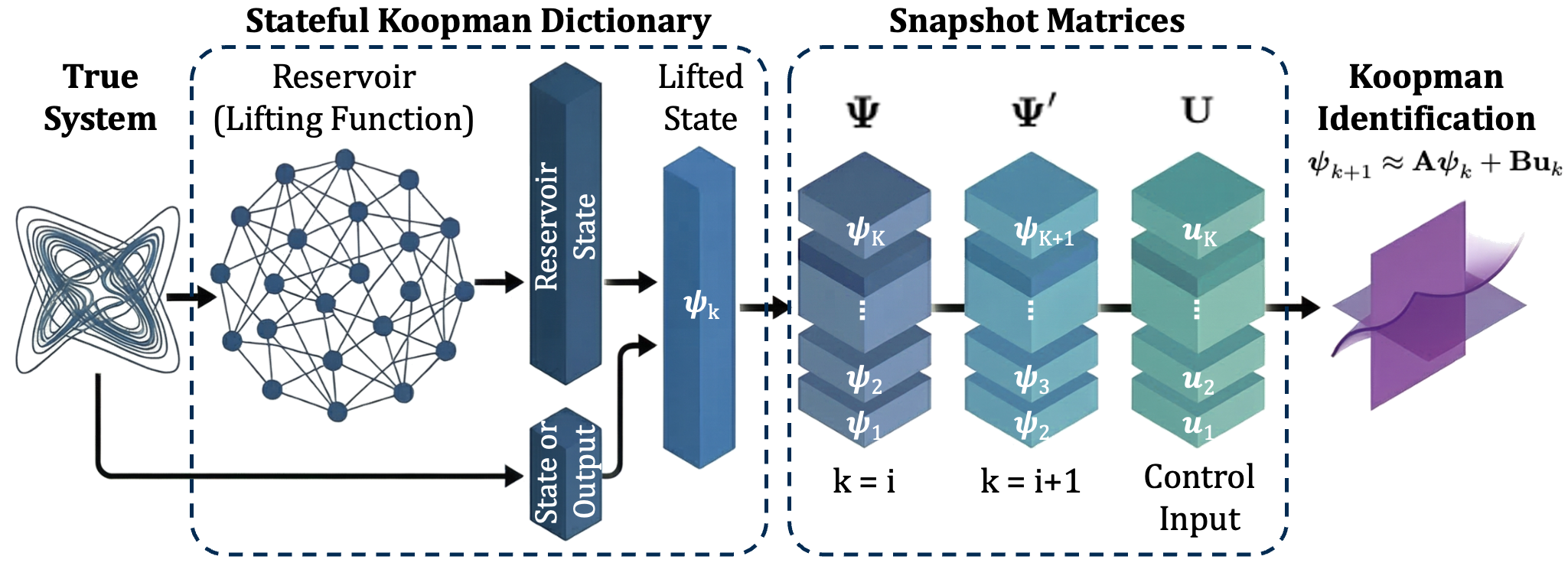}
    % \caption{Proposed RC--Koopman framework.}
    \caption{\rev{\textbf{Schematic of the proposed RC--Koopman framework.} System measurements are mapped to a high-dimensional reservoir that serves as a stateful dictionary, where the resulting reservoir state is augmented by the system state or output to form the lifted state $\boldsymbol{\psi}_k$. The Koopman matrices $\mathbf{A}$ and $\mathbf{B}$ are then identified via least-squares regression using time-shifted snapshot matrices ($\mathbf{\Psi}, \mathbf{\Psi}'$) and control inputs ($\mathbf{U}$).}}
    \label{fig:RC-Koopman-Paper}
\end{figure}

Several recent studies have explored connections between the \ac{RC} paradigm and Koopman-based modeling~\cite{bollt2021explaining,gauthier2021next,gulina2021two,yilmaz2024model}. For example, it was shown in~\cite{bollt2021explaining} that linear reservoirs with identity activation functions are functionally equivalent to \ac{HDMD}, linking reservoir dynamics to delay-coordinate approximations of the Koopman operator. Extensions incorporating nonlinear transformations of delay embeddings were later proposed in~\cite{gauthier2021next}. Other works combine reservoirs with \ac{EDMD}-type constructions~\cite{gulina2021two} or integrate reservoir features into Hankel-based frameworks~\cite{yilmaz2024model}. While these studies demonstrate empirical compatibility between reservoir dynamics and Koopman lifting, the role of reservoir properties such as contraction and spectral structure in the resulting Koopman identification remains insufficiently understood. %Moreover, most validations are restricted to noise-free synthetic chaotic systems, leaving open questions regarding robustness under measurement noise and experimental data.

In this work, we propose a Koopman identification framework (Fig.~\ref{fig:RC-Koopman-Paper}) that interprets the reservoir as a \emph{stateful Koopman dictionary}. The recurrent dynamics lift system histories into a high-dimensional fading-memory space, enabling expressive representations while maintaining favorable numerical conditioning. The contraction properties induced by the \ac{ESP} naturally compress temporal information into a stable latent representation, mitigating the dimensionality explosion and ill-conditioning commonly encountered in delay-coordinate liftings. Our main contributions are as follows:
\begin{itemize}
    \item We formalize a unified \ac{RC}--Koopman framework and establish how \ac{ESP} ensures well-posedness and favorable numerical conditioning of the lifted approximation.
    \item We characterize how finite reservoir memory determines Koopman eigenfunction observability and provide a correlation-based algorithm for spectral radius selection.
\end{itemize}

The remainder of this paper is organized as follows. Section~\ref{sec:rc_koopman_framework} introduces the proposed \ac{RC}--Koopman identification framework. Section~\ref{sec:theoretical_analysis} presents the theoretical analysis of the lifting properties. Section~\ref{sec:results} evaluates the approach on synthetic benchmarks. Finally, Section~\ref{sec:conclusion} concludes the paper and outlines future directions.

\section{The RC--Koopman Framework}
\label{sec:rc_koopman_framework}

\subsection{Problem Formulation}

Consider the discrete-time nonlinear system
\begin{subequations}\label{eq:nonlinear_system}
    \begin{align}
    \mathbf{x}_{k+1} &= \mathbf{f}(\mathbf{x}_k,\mathbf{u}_k), \label{eq:nonlinear_state} \\
    \mathbf{y}_k &= \mathbf{h}(\mathbf{x}_k),
    \label{eq:nonlinear_output}
\end{align}
\end{subequations}
where $\mathbf{x}_k \in \mathbb{R}^{n_x}$ denotes the system state, $\mathbf{u}_k \in \mathbb{R}^{n_u}$ the control input, and $\mathbf{y}_k \in \mathbb{R}^{n_y}$ the measured output. The nonlinear state transition map $\mathbf{f}$ and the measurement function $\mathbf{h}$ are assumed unknown. Our objective is to construct a finite-dimensional linear representation of the dynamics from data that admits an operator-theoretic interpretation and is suitable for prediction and control.

% Koopman operator theory provides a principled framework for this purpose by representing nonlinear dynamics as linear evolution of observable functions.
% in an infinite-dimensional Hilbert space $\mathcal{H}$. In practice, finite-dimensional approximations are obtained by identifying a subspace $\bar{\mathcal{H}}\subset\mathcal{H}$ that is approximately invariant under the Koopman operator. 
For systems with control inputs, we adopt the commonly used approximation employed in Koopman-with-control formulations such as EDMDc~\cite{proctor2016dynamic}, in which the lifted dynamics are modeled as affine in the control input,
\begin{equation}
\boldsymbol{\psi}_{k+1}
\approx
\mathbf{A}\boldsymbol{\psi}_k
+
\mathbf{B}\mathbf{u}_k ,
\end{equation}
where $\boldsymbol{\psi}_k$ denotes the lifted state and $\mathbf{A},\mathbf{B}$ are finite-dimensional operators to be identified from data (Fig.~\ref{fig:RC-Koopman-Paper}). The central design choice in this framework is therefore the
construction of the lifting functions $\boldsymbol{\psi}_k$.

\subsection{Stateful Lifting via Reservoir Dynamics}

Rather than employing stateless dictionaries as in standard \ac{EDMD}, we construct a stateful lifting using a reservoir system driven by measurements of~\eqref{eq:nonlinear_system}. The reservoir state $\mathbf{r}_k\in\mathbb{R}^{n_r}$ evolves according to
\begin{equation}
    \mathbf{r}_k
    =
    \sigma\!\left(
        \mathbf{W}_{\mathrm{res}}\mathbf{r}_{k-1}
        +
        \mathbf{W}_{\mathrm{in}}\mathbf{v}_k
    \right),
    \label{eq:reservoir_dynamics}
\end{equation}
where $\sigma$ is a nonlinear activation function, $\mathbf{v}_k \in \mathbb{R}^{n_v}$ is the reservoir input, and $\mathbf{W}_{\mathrm{res}}\in\mathbb{R}^{n_r\times n_r}$ and $\mathbf{W}_{\mathrm{in}}\in\mathbb{R}^{n_r\times n_v}$ are fixed reservoir and input weight matrices, respectively. For simplicity, bias terms are omitted, although they can be incorporated without affecting the analysis. Throughout this work, we set $\mathbf{v}_k=\mathbf{y}_k$\rev{, therefore implying $n_v = n_y$.}

As illustrated in Fig.~\ref{fig:RC-Koopman-Paper}, the lifted state is defined as
\begin{equation}
    \boldsymbol{\psi}_k
    =
    [\mathbf{v}_k^\top \;
        \mathbf{r}_k^\top]^\top
    \in\mathbb{R}^{n_\psi},
    \qquad
    n_\psi = n_v + n_r ,
    \label{eq:lifted_state}
\end{equation}
\rev{where we typically choose $n_r \gg n_v$, ensuring that the nonlinear system dynamics are projected into a sufficiently rich feature space.}
This defines a nonlinear observable mapping from the current measurement and reservoir state to a finite-dimensional feature vector. Although the observable depends only on the current measurement $\mathbf{v}_k$, the reservoir state $\mathbf{r}_k$ implicitly summarizes past measurements through the recurrent dynamics in~\eqref{eq:reservoir_dynamics}. Consequently, temporal dependencies are encoded implicitly through the reservoir recurrence rather than explicitly through stacked delays, as in Hankel-based liftings. Moreover, when the system state is not directly measurable, the reservoir effectively constructs a history-dependent embedding of the measurement sequence, analogous in spirit to delay-coordinate embeddings used in Hankel-based Koopman methods.

\subsection{Echo State Property}

A fundamental requirement for the above lifting to be well-posed is that the reservoir state be uniquely determined by the input history. This is guaranteed by the \ac{ESP} from the \ac{RC} paradigm, which ensures that, for any bounded input sequence, the reservoir state is uniquely determined by the input history and becomes independent of the initial reservoir condition. Under the following assumption, the reservoir update defines a contractive, input-driven dynamical system.
\begin{assumption}[Sufficient Condition for ESP]
\label{assump:esp}
The activation function $\sigma$ in~\eqref{eq:reservoir_dynamics} is globally
Lipschitz with constant $L_\sigma$, and
\begin{equation}
    L_\sigma \|\mathbf{W}_{\mathrm{res}}\|_2 < 1,
    \label{eq:esp_condition}
\end{equation}
where $\|\cdot\|_2$ denotes the spectral norm.
\end{assumption}
% Under Assumption~\ref{assump:esp}, the reservoir update is a contraction mapping with respect to the state variable, ensuring incremental global asymptotic stability and uniqueness of the reservoir trajectory for any bounded input sequence.
\begin{assumption}[Practical ESP Scaling Rule]
\label{assump:practical_esp}
The reservoir weight matrix in~\eqref{eq:reservoir_dynamics} is scaled such that
\begin{equation}
\label{eq:practical_esp}
\rho(\mathbf W_{\mathrm{res}}) < 1, 
\end{equation}
where $\rho(\cdot)$ denotes the spectral radius.
This condition is widely used in practice with \texttt{tanh} activation
function.
\end{assumption}
\begin{remark}
    Condition~\eqref{eq:practical_esp} is weaker than the sufficient condition~\eqref{eq:esp_condition}, as $\rho(\mathbf{W}) \le \|\mathbf{W}\|_2$. \rev{While~\eqref{eq:esp_condition} provides a rigorous contraction guarantee, it is often overly conservative, limiting the reservoir’s memory and expressivity. In practice, the spectral radius condition~\eqref{eq:practical_esp} is the standard \ac{RC} design rule; it facilitates richer dynamics where the activation function's nonlinearity typically ensures empirical stability.} Accordingly, we distinguish between the theoretical well-posedness of Assumption~\ref{assump:esp} and the practical scaling of Assumption~\ref{assump:practical_esp}, used here to balance lifting richness with stability.
\end{remark}

% \begin{remark}
% The condition~\eqref{eq:practical_esp} is weaker than the sufficient condition~\eqref{eq:esp_condition}, since for any matrix $\mathbf{W}$ we have $\rho(\mathbf{W}) \le \|\mathbf{W}\|_2$. 
% \rev{While~\eqref{eq:esp_condition} offers a rigorous contraction-based guarantee, it is often overly conservative, unnecessarily restricting the reservoir's memory capacity and dynamical expressivity. In contrast, the spectral radius condition~\eqref{eq:practical_esp} is the established design rule in \ac{RC} paradigm. It allows for a broader range of rich, non-conservative dynamics where the nonlinearity of the activation function often induces the necessary contraction in practice.}
% Throughout this work, we distinguish between the sufficient theoretical condition (Assumption~\ref{assump:esp}), which ensures well-posedness, and the practical scaling rule (Assumption~\ref{assump:practical_esp}), which we use to balance lifting richness with empirical stability.
% \end{remark}

\subsection{Koopman Operator Identification}

\rev{As defined in~\eqref{eq:reservoir_dynamics} and~\eqref{eq:lifted_state}, the lifted state $\boldsymbol{\psi}_k$ comprises the system measurement $\mathbf{y}_k$ and the stateful reservoir features $\mathbf{r}_k$.} 
Using the data sequence $\{(\boldsymbol{\psi}_k,\boldsymbol{\psi}_{k+1},\mathbf{u}_k)\}_{k=1}^K$ over $K$ time steps, we form the snapshot matrices (Fig.~\ref{fig:RC-Koopman-Paper})
\begin{subequations}\label{eq:identification_data_pairs}
    \begin{align}
    \boldsymbol{\Psi} &= [\boldsymbol{\psi}_1,\dots,\boldsymbol{\psi}_K]
    \in \mathbb{R}^{n_\psi \times K}, \label{eq:lifted_data_matrices_current}\\
    \boldsymbol{\Psi}' &= [\boldsymbol{\psi}_2,\dots,\boldsymbol{\psi}_{K+1}]
    \in \mathbb{R}^{n_\psi \times K}, \label{eq:lifted_data_matrices_next}\\
    \mathbf{U} &= [\mathbf{u}_1,\dots,\mathbf{u}_K]
    \in \mathbb{R}^{n_u \times K}. \label{eq:control_data_matrices}
\end{align}
\end{subequations}
Note that the control input is not lifted and therefore enters the identified model affinely. A finite-dimensional Koopman approximation is obtained by solving
\begin{equation}
\label{eq:Koopman_control_affine_form}
    \boldsymbol{\Psi}'
    \approx
    \mathbf{A}\boldsymbol{\Psi} + \mathbf{B}\mathbf{U},
    \qquad
    \mathbf{A} \in \mathbb{R}^{n_\psi \times n_\psi}, \;
    \mathbf{B} \in \mathbb{R}^{n_\psi \times n_u},
\end{equation}
in the least-squares sense, yielding
\begin{equation}
    [\mathbf{A}\;\mathbf{B}]
    =
    \boldsymbol{\Psi}'
    \begin{bmatrix}
        \boldsymbol{\Psi} \\
        \mathbf{U}
    \end{bmatrix}^{\!\dagger},
    \label{eq:koopman_regression}
\end{equation}
where $(\cdot)^\dagger$ denotes the Moore--Penrose pseudoinverse.
In practice, regularization (e.g., ridge regression) may be incorporated to
improve numerical robustness. % (see Appendix~\ref{subsec:practical}).
Since $\mathbf{v}_k=\mathbf{y}_k$ is explicitly included in the lifted state, the system output can be recovered through a simple projection
\begin{equation}
    \hat{\mathbf{y}}_k
    =
    \mathbf{C}\boldsymbol{\psi}_k,
    \qquad
    \mathbf{C}
    =
    \begin{bmatrix}
        \mathbf{I}_{n_v} & \mathbf{0}_{n_r}
    \end{bmatrix},
\end{equation}
yielding a one-step prediction with measurement updates. \rev{\begin{remark}
    Unlike many \ac{RC} applications focused on multi-step forecasting, our approach utilizes the reservoir as a stateful dictionary for operator-theoretic identification, adhering to the standard Koopman formulation.
\end{remark}}
% , summarized in Algorithm~\ref{alg:single_step_prediction}. 
% \rev{It is important to note that \eqref{eq:Koopman_control_affine_form} focuses on identifying the one-step-ahead linear transition in the lifted space. This is consistent with the standard Koopman operator formulation, which seeks a linear propagator for a single time increment. The primary value of employing a reservoir here is not to alter the prediction horizon, but to provide a stateful, systematically generated dictionary that avoids the heuristic nature of stateless observables in \ac{EDMD} while maintaining superior numerical conditioning over delay-coordinate embeddings.}

%%%%%% TODO: REMOVE DUE TO SPACE LIMIT
% \begin{algorithm}
% \caption{Single-Step RC--Koopman Prediction}
% \label{alg:single_step_prediction}
% \small
% \KwIn{
% Identified operators $\mathbf{A}, \mathbf{B}$;
% reservoir weights $\mathbf{W}_{\mathrm{res}}, \mathbf{W}_{\mathrm{in}}$;
% current reservoir state $\mathbf{r}_{k-1}$;
% measurement $\mathbf{y}_k$;
% input $\mathbf{u}_k$
% }
% \KwOut{Predicted output $\hat{\mathbf{y}}_{k+1}$}

% Set reservoir input $\mathbf{v}_k \gets \mathbf{y}_k$\;
% Update reservoir state using \eqref{eq:reservoir_dynamics}\;
% Construct lifted state $\boldsymbol{\psi}_k$ using \eqref{eq:lifted_state}\;
% Predict lifted state using \eqref{eq:Koopman_control_affine_form}\;
% Recover output $\hat{\mathbf{y}}_{k+1} = \mathbf{C}\boldsymbol{\psi}_{k+1}$\;

% \Return $\hat{\mathbf{y}}_{k+1}$
% \end{algorithm}

%%%%%%%%%%%%%%%%%%%%%%%%%%%%%%%%%%%%%%%%%%%%%%%%%%%%%%%%%%%%%%%%%%%%%%%%%%%%%%%%
\section{Theoretical Analysis of Reservoir Liftings}
\label{sec:theoretical_analysis}

Consider the reservoir system~\eqref{eq:reservoir_dynamics} and the lifted state defined as~\eqref{eq:lifted_state}, we clarify hereafter the theoretical role of reservoir dynamics in Koopman-based identification. 
% Rather than introducing reservoir features merely as an empirical lifting, we analyze how structural properties of \ac{RC} paradigm characterize well-posedness, conditioning of lifted data, temporal memory, and Koopman spectrum.

\subsection{Well-Posedness of Reservoir Liftings}

We first establish that the reservoir state is uniquely determined by the input sequence under the \ac{ESP}. Moreover, the lifted coordinates $\boldsymbol\psi_k$ are guaranteed to evolve on a bounded invariant set determined solely by the input sequence. Such well-posedness is not intrinsic to data-driven liftings such as \ac{EDMD} or delay embeddings. Let $\ell_\infty(\mathbb Z,\mathbb R^{n_r})$ denote the Banach space of bounded sequences with norm
\begin{equation}
    \|\mathbf r\|_{\ell_\infty} = \sup_{k\in\mathbb Z} \|\mathbf r_k\|_2 .
\end{equation}

\begin{theorem}[Well-Posed Reservoir Lifting]
\label{thm:esp_lifting}
Assume the activation function $\sigma$ is Lipschitz continuous with
constant $L_\sigma$ and sufficient condition for ESP~\eqref{eq:esp_condition} is satisfied. Then for any bounded input sequence $\{\mathbf v_k\}\in\ell_\infty(\mathbb Z,\mathbb R^{n_v})$,
the reservoir system \eqref{eq:reservoir_dynamics} admits a unique bounded state trajectory $\{\mathbf r_k\}\in\ell_\infty(\mathbb Z,\mathbb R^{n_r})$ independent of the initial condition. Moreover, the induced mapping from input histories to reservoir states is causal and has the fading-memory property.
\end{theorem}

\begin{proof}
Define the operator $\mathcal T$ acting on sequence space
$\ell_\infty(\mathbb Z,\mathbb R^{n_r})$ as 
$(\mathcal T \mathbf r)_k
=
\sigma\!\left(
\mathbf W_{\mathrm{res}} \mathbf r_{k-1}
+
\mathbf W_{\mathrm{in}} \mathbf v_k
\right).$
For any two sequences $\mathbf r,\tilde{\mathbf r}$, 
$\|(\mathcal T\mathbf r)_k-(\mathcal T\tilde{\mathbf r})_k\|_2
\le
L_\sigma
\|\mathbf W_{\mathrm{res}}\|_2
\|\mathbf r_{k-1}-\tilde{\mathbf r}_{k-1}\|_2 .$
Taking the supremum over $k$ gives
$\|\mathcal T\mathbf r-\mathcal T\tilde{\mathbf r}\|_{\ell_\infty}
\le
L_\sigma \|\mathbf W_{\mathrm{res}}\|_2
\|\mathbf r-\tilde{\mathbf r}\|_{\ell_\infty}.$
Under condition \eqref{eq:esp_condition}, $\mathcal T$ is a contraction mapping. Existence and uniqueness of a fixed point therefore follow from the Banach fixed-point theorem. The resulting state sequence depends causally on the input history and exhibits exponential forgetting of initial conditions, which establishes the fading-memory property \cite{jaeger2001echo}.
\end{proof}

%%%%%%%%%%%%%%%%%%%%%%%%%%%%%%%%%%%%%%%%%%%%%%%%%%%%
\subsection{Conditioning of Reservoir Lifted Data}

Let the lifted data matrix be defined by~\eqref{eq:lifted_data_matrices_current} and define the empirical Gramian
\begin{equation}
\mathbf G
=
\boldsymbol\Psi\boldsymbol\Psi^\top
=
\sum_{k=1}^K
\boldsymbol\psi_k\boldsymbol\psi_k^\top .
\label{eq:gramian_analysis}
\end{equation}
\rev{The following proposition shows that given the persistence of excitation condition~\cite{Boddupalli2019}, Gram matrix does not grow unboundedly with the sample length $K$. Consequently, the resulting condition number bound remains $K$-independent, which explains the empirical observation that RC--Koopman maintains high numerical stability even as the dataset size or reservoir dimension increases.}

\begin{proposition}[Conditioning of Reservoir Lifted Data]
\label{prop:rc_conditioning_analysis}

Assume $\|\mathbf v_k\|_2 \le M_v$ for all $k$. Under Theorem~\ref{thm:esp_lifting}, there exists a constant
$C_\psi>0$ such that
$
\|\boldsymbol\psi_k\|_2 \le C_\psi .
$
Hence,
\begin{equation}
    \lambda_{\max}(\mathbf G) \le C_\psi^2 K .
\end{equation}
Furthermore, if the lifted sequence satisfies the persistent excitation condition
\begin{equation}
\frac{1}{K}
\sum_{k=1}^K
\boldsymbol\psi_k \boldsymbol\psi_k^\top
\succeq
\alpha I,
\qquad
\alpha>0,
\label{eq:pe_analysis}
\end{equation}
then
\begin{equation}
    \kappa(\mathbf G) \le \frac{C_\psi^2}{\alpha},
\end{equation}
which is independent of the sample length $K$.
\end{proposition}

\begin{proof}
From reservoir dynamics~\eqref{eq:reservoir_dynamics} and Lipschitz continuity, 
$\|\mathbf r_k\| \le L_\sigma (\|\mathbf W_{\mathrm{res}}\|_2\|\mathbf r_{k-1}\| + \|\mathbf W_{\mathrm{in}}\|_2 M_v).$
Let
$
a=L_\sigma\|\mathbf W_{\mathrm{res}}\|_2, \,
b=L_\sigma\|\mathbf W_{\mathrm{in}}\|_2 M_v .
$
Since $a<1$ (Assumption~\ref{assump:esp}),
$
\|\mathbf r_k\|_2
\le
a^k \|\mathbf r_0\|_2
+
\frac{b}{1-a}.
$
Hence $\mathbf r_k$ is uniformly bounded.
Define
$
C_\psi^2 = M_v^2 + (b/(1-a))^2,
$
thus
$
\|\boldsymbol\psi_k\|_2^2 = \|\mathbf v_k\|_2^2 + \|\mathbf r_k\|_2^2 \le C_\psi^2.
$
Since $\mathbf G$ is symmetric and positive semidefinite,
$
\lambda_{\max}(\mathbf G)
=
\|\mathbf G\|_2
\le
\sum_{k=1}^K
\|\boldsymbol\psi_k\boldsymbol\psi_k^\top\|_2
\le
C_\psi^2 K.
$
The lower bound follows directly from
\eqref{eq:pe_analysis}.
\end{proof}

%This establishes that reservoir-induced liftings admit conditioning bounds derived from system structure, rather than from assumptions on invariant measures or delay lengths. Moreover, as $\rho(\mathbf{W}_{\mathrm{res}}) \to 0$, the reservoir dynamics become strongly contractive, and the states $\mathbf{r}_k$ are dominated by the instantaneous input $\mathbf{y}_k$. Consequently, the lifted coordinates exhibit increased collinearity, leading to a more regularized (smoother) function class with reduced approximation capacity. In contrast, as $\rho(\mathbf{W}_{\mathrm{res}}) \to 1^{-}$, temporal dependencies persist over longer horizons and reservoir units become less correlated. This reduces redundancy in the lifted features and can improve the conditioning of the empirical Gram matrix $\mathbf{G}$. Operating in this regime approaches the so-called ``edge of stability'' (or more precisely, ``edge of chaos'')~\cite{carroll2020reservoir}, where reservoirs are frequently observed to exhibit enhanced computational capacity~\cite{parlitz2024learning}.

%%%%%%%%%%%%%%%%%%%%%%%%%%%%%%%%%%%%%%%%%%%%%%%%%%%%
\subsection{Temporal Memory Structure}

Reservoir liftings encode temporal information through exponentially fading memory. 
% This exponentially decaying memory contrasts with delay-coordinate embeddings, which impose a hard truncation at a fixed window length. 
\rev{We argue that this allows the identified operator by RC--Koopman to approximate Koopman eigenfunctions that are functionals of the system's history, effectively enriching the spectral decomposition of the underlying nonlinear flow.}

\begin{proposition}[Exponential Memory Decay]
\label{prop:memory_analysis}

Let $\gamma = L_\sigma\|\mathbf W_{\mathrm{res}}\|_2 .$ Under condition $\gamma<1$ (Assumption~\ref{assump:esp}), the sensitivity of the reservoir state to past inputs satisfies
\begin{equation}
    \left\| \frac{\partial \mathbf r_k}{\partial \mathbf v_{k-\tau}} \right\|_2
    \le
    \gamma^\tau \|\mathbf W_{\mathrm{in}}\|_2 .
\end{equation}
Consequently the effective memory horizon for precision $\epsilon>0$ satisfies
\begin{equation}
    \tau_\epsilon
    \ge
    \frac{\log(\epsilon/\|\mathbf W_{\mathrm{in}}\|_2)}{\log\gamma}.
    \label{eq:effective_memory}
\end{equation}
\end{proposition}

\begin{proof} Considering the reservoir dynamics~\eqref{eq:reservoir_dynamics} and applying the chain rule, the sensitivity of the state at step $k$ to an input at step $k-\tau$ is given by the product of Jacobians along the trajectory:
$\frac{\partial \mathbf{r}_k}{\partial \mathbf{v}_{k-\tau}} =
\left( \prod_{j=1}^{\tau} \frac{\partial \mathbf{r}_{k-j+1}}{\partial \mathbf{r}_{k-j}} \right) \frac{\partial \mathbf{r}_{k-\tau}}{\partial \mathbf{v}_{k-\tau}}.$
The Jacobian of the reservoir update with respect to the state is $\mathbf{D}_j \mathbf{W}_{\mathrm{res}}$, where $\mathbf{D}_j = \mathrm{diag}(\sigma'(\cdot))$ is the diagonal matrix of activation derivatives. Since $\sigma$ is $L_\sigma$-Lipschitz (Assumption~\ref{assump:esp}), $\|\mathbf{D}_j\|_2 \le L_\sigma$. Similarly, $\frac{\partial \mathbf{r}_{k-\tau}}{\partial \mathbf{v}_{k-\tau}} = \mathbf{D}_{k-\tau} \mathbf{W}_{\mathrm{in}}$. Taking spectral norms yields
\begin{align*}
\tiny
\Bigg|\Bigg| \frac{\partial \mathbf{r}_k}{\partial \mathbf{v}_{k-\tau}} \Bigg|\Bigg|_2
&\le \left( \prod_{j=1}^{\tau} ||\mathbf{D}_{k-j+1}||_2 ||\mathbf{W}_{\mathrm{res}}||_2 \right) ||\mathbf{D}_{k-\tau}||_2 ||\mathbf{W}_{\mathrm{in}}||_2 \\
&\le (L_\sigma ||\mathbf{W}_{\mathrm{res}}||_2)^\tau L_\sigma ||\mathbf{W}_{\mathrm{in}}||_2.
\end{align*}
Since $L_\sigma = 1$ for \texttt{tanh} activation, we obtain the simplified bound $\gamma^\tau \|\mathbf{W}_{\mathrm{in}}\|_2$. Setting the upper bound equal to $\epsilon$ and solving for $\tau$ yields \eqref{eq:effective_memory}.
\end{proof}

\begin{remark}
\label{prop:hankel_memory}
This result implies that while the reservoir can theoretically hold a memory indefinitely\footnote{As $\gamma \to 1$, the effective horizon $\tau_{\epsilon} \to \infty$.}, its practical memory is finite and tunable via the spectral properties of $\mathbf{W}_{\mathrm{res}}$. By contrast, Hankel liftings impose a uniform dependence and hard truncation on a finite window.
A Hankel lifting of length $d$ for temporal embedding produces observables that depend uniformly on the finite window $\{\mathbf{y}_{k-d+1},\dots,\mathbf{y}_k\}$. The memory is strictly zero for lags $\tau \ge d$. 
% To achieve an effective memory depth equivalent to a reservoir with contraction $\gamma \approx 0.99$ (at precision $\epsilon=10^{-3}$), a Hankel lifting would require a window length of $d \approx \tau_{\epsilon} \approx 687$\footnote{Assuming normalized weights $\|\mathbf{W}_{\mathrm{in}}\|_2 \approx 1$, we solve $0.99^\tau = 10^{-3}$. Taking common logs, $\tau \log(0.99) = \log(10^{-3})$ yields $\tau \approx 687$.}. This would result in a lifting dimension $n_y \cdot d$ that is computationally intractable for SVD-based \ac{HAVOK} methods. The RC--Koopman framework compresses this long-term history into a fixed dimension $n_r$, effectively decoupling the memory depth from the state-space dimension.
\end{remark}

%%%%%%%%%%%%%%%%%%%%%%%%%%%%%%%%%%%%%%%%%%%%%%%%%%%%
\subsection{Reservoir Memory and Koopman Spectrum}
\label{subsec:koopman_memory}

% The previous analysis established that the reservoir state possesses a finite fading memory characterized by the contraction factor $\gamma$. 
% In particular, Proposition~\ref{prop:memory_analysis} shows that the sensitivity of the reservoir state to past inputs decays exponentially, so that temporal information older than the effective horizon $\tau_\epsilon$ contributes at most $\epsilon$ to the lifted state. This implies that the reservoir lifting~\eqref{eq:lifted_state} acts as a nonlinear temporal feature map that aggregates information from a finite window of past observations. 
% We next show that this memory horizon determines the set of Koopman eigenmodes that can be resolved by the lifted observable space.

\begin{theorem}[Koopman Spectral Observability]
\label{thm:koopman_resolution}

Consider the discrete-time system~\eqref{eq:nonlinear_system} and assume the reservoir satisfies the ESP condition under Assumption~\ref{assump:esp}. Let $\tau_\epsilon$ denote the effective memory horizon defined in
Proposition~\ref{prop:memory_analysis}.
Then the lifted observable space generated by~\eqref{eq:lifted_state} can \rev{approximate} Koopman eigenfunctions $\phi_j$ with eigenvalues
$\lambda_j$ satisfying
\begin{equation}
|\lambda_j|^{\tau_\epsilon} \gtrsim \epsilon .
\label{eq:koopman_resolution_condition}
\end{equation}
\rev{Here, $a \gtrsim b$ denotes $a \ge C b$ for some constant $C>0$ independent of $j$ and $\tau_\epsilon$.} \rev{Koopman eigenfunctions (or equivalently, spectral components)} satisfying
\begin{equation}
|\lambda_j|
\gtrsim
\exp\!\left(\frac{\log \epsilon}{\tau_\epsilon}\right)
\end{equation}
\rev{remain persistent within the reservoir memory horizon and can be retained in the lifted observables}, whereas modes with faster decay vanish within the effective
memory horizon. Consequently, increasing the spectral radius $\rho(\mathbf W_{\mathrm{res}})$ enlarges $\tau_\epsilon$ and increases the range of Koopman \rev{eigenfunction components} whose temporal correlations remain observable within the lifted feature space.
\end{theorem}
\begin{proof} Koopman eigenfunctions evolve according to
$
\phi_j(\mathbf y_k) = \lambda_j^k \phi_j(\mathbf y_0).
$
\rev{Recovering such eigenfunction components from time-series observables} requires temporal correlations over a time span comparable to the decay rate of the corresponding \rev{spectral component}. If
$
|\lambda_j|^{\tau_\epsilon} \ll \epsilon ,
$
the temporal correlations associated with this \rev{eigenfunction component} vanish within
the reservoir memory horizon and therefore cannot be recovered from
the lifted observables. Conversely, \rev{eigenfunction components} satisfying
$
|\lambda_j|^{\tau_\epsilon} \gtrsim \epsilon
$
remain persistent within the observable time window and can be captured by the reservoir lifting. Substituting the expression for $\tau_\epsilon$ from Proposition~\ref{prop:memory_analysis} yields the stated spectral resolution condition.
\end{proof}

\begin{remark}
Theorem~\ref{thm:koopman_resolution} characterizes Koopman eigenfunction observability within finite reservoir memory but does not guarantee an invariant subspace. It establishes a spectral trade-off governed by $\rho(\mathbf{W}_{\mathrm{res}})$: a smaller $\rho(\mathbf{W}_{\mathrm{res}})$ enhances contraction and numerical stability but restricts the temporal depth $\tau_\epsilon$, potentially filtering slow-decaying modes. Conversely, increasing $\rho(\mathbf{W}_{\mathrm{res}})$ expands the observable spectral range at the expense of reduced contraction margins.

% The theorem characterizes which Koopman eigenfunctions remain observable given the reservoir's finite memory but does not guarantee an invariant Koopman subspace, a central and longstanding challenge in Koopman theory. \rev{Moreover, it provides a spectral interpretation of the trade-off between the reservoir's contraction strength and its dynamical expressivity. This balance is fundamentally governed by $\rho(\mathbf{W}_{\mathrm{res}})$. While a smaller spectral radius ensures faster state convergence (stronger contraction) and numerical stability, it restricts the temporal depth $\tau_\epsilon$, potentially omitting slow-decaying Koopman modes critical for accurate reconstruction. Conversely, increasing $\rho(\mathbf{W}_{\mathrm{res}})$ toward unity expands the observable spectral range, improving the linear approximation of the Koopman operator at the cost of reduced contraction margins.}
\end{remark}

% Theorem~\ref{thm:koopman_resolution} provides a practical design interpretation for reservoir lifting. Specifically, the spectral radius of the reservoir determines the effective temporal resolution of the observable space and therefore \rev{influences which Koopman eigenfunction components can be captured.} 
Algorithm~\ref{alg:spectral_radius_correlation} translates the theoretical connection between reservoir memory and Koopman spectral resolution into a practical design rule. Instead of selecting the spectral radius heuristically, the reservoir memory may be tuned to match the dominant dynamical time scales of the system. Specifically, since Koopman eigenvalues govern the decay of temporal correlations in the observed dynamics, the dominant system time scale could be estimated directly from the autocorrelation function of the measurements. In particular, the correlation time $\tau_c$ is defined as the $e$-folding time of the normalized autocorrelation function, i.e., the smallest lag satisfying $C(\tau_c) \le e^{-1}$. Matching this correlation decay with the reservoir memory decay provides a practical guideline for spectral radius selection. 

\begin{algorithm}
\caption{Spectral Radius Selection}
\label{alg:spectral_radius_correlation}
\small
\KwIn{
    Observation trajectory $\{\mathbf y_k\}_{k=1}^{K}$;
    correlation threshold (default $e^{-1}$)
}
\KwOut{
    Selected spectral radius $\rho(\mathbf W_{\mathrm{res}})$;
    estimated correlation time $\tau_c$
}

\KwData{
    maximum lag $L$ for autocorrelation computation %(e.g., $L = K/4$)
}

Collect a trajectory $\{\mathbf y_k\}_{k=1}^{K}$ as input data\;

Estimate the normalized autocorrelation for lags $\tau = 1, \ldots, L$:
$
C(\tau) = (\sum_{k=1}^{K-\tau} \mathbf y_k^\top \mathbf y_{k+\tau})/(\sum_{k=1}^{K} \mathbf y_k^\top \mathbf y_k)
$\;

Determine the correlation time as the smallest lag where autocorrelation drops below $e^{-1}$:
$
\tau_c = \min \{\tau : C(\tau) \le e^{-1}\}
$\;

Choose the reservoir spectral radius using the exponential decay relation:
$
\rho(\mathbf W_{\mathrm{res}}) \approx \exp\!\left(-1/\tau_c\right)
$\;

Initialize reservoir weights with this spectral radius\;

Perform Koopman identification using the lifted observables from the reservoir states\;

\Return $\rho(\mathbf W_{\mathrm{res}})$, $\tau_c$
\end{algorithm}

%%%%%%%%%%%%%%%%%%%%%%%%%%%%%%%%%%%%%%%%%%%%%%%%%%%%%%%%%%%%%%%%%%%%%%%%%%%%%%%%
\section{Results and Discussion}
\label{sec:results}
The RC--Koopman framework is evaluated on two systems \rev{assuming full-state measurements ($\mathbf{y}_k = \mathbf{x}_k$)}: (i) an unforced \textit{Duffing oscillator} ($\dot{x}_1 = x_2, \, \dot{x}_2 = -c_1 x_2 - (c_2 x_1^2 + c_3)x_1$, where $c_1=0.5, c_2=1, c_3=-1$) and (ii) a controlled \textit{differential-drive robot} ($\dot{x} = v\cos\theta, \dot{y} = v\sin\theta, \dot{\theta} = \omega$, where $x, y$ denotes planar position, $\theta$ the heading angle, $v, \omega$ the linear and angular velocities) integrated via forward Euler ($\Delta t = 0.05$). We compare RC--Koopman against \ac{EDMD} \rev{with \ac{RBF} dictionary} and \ac{HAVOK}\rev{\footnote{Our \ac{HAVOK} implementation utilizes the delay-embedding procedure from \cite{brunton2017chaos} but performs identification via the same least-squares regression used for \ac{EDMD} and RC--Koopman, making it methodologically equivalent to \ac{HDMD}~\cite{bollt2021explaining}. This ensures that performance variances are strictly attributable to the dictionary architecture
% —static nonlinear (\ac{EDMD}), linear memory (\ac{HAVOK}), or nonlinear stateful (RC--Koopman)—
rather than the underlying solver.}} using a fixed lifting dimension of 12 and identical regularization. Beyond prediction accuracy, we analyze numerical conditioning and spectral properties to verify the theoretical results in Section~\ref{sec:theoretical_analysis}.

% %%%%%%%%%%%%%%%%%%%%%%%%%%%%%%%%%%%%%%%%%%%%%%%%%%%%
% \subsection{Evaluation Metrics}

% Prediction accuracy is quantified using the \ac{RMSE}:
% $
% \mathrm{RMSE}(k)
% =
% \sqrt{\frac{1}{n_y}\|\hat{\mathbf y}_k - \mathbf y_k\|_2^2},
% $
% evaluated for one-step prediction over the test trajectories. Numerical conditioning is assessed through the condition number of the lifted data Gramian 
% $
% \kappa(\boldsymbol{\Psi}\boldsymbol{\Psi}^\top),
% $
% which reflects the stability of the least-squares estimation used to identify the Koopman operator. We further analyze the spectral properties of the learned operator via the spectral radius $\rho(\mathbf A)$. In particular, stable and well-conditioned identification should yield operators whose spectral characteristics remain consistent across training trials.

%%%%%%%%%%%%%%%%%%%%%%%%%%%%%%%%%%%%%%%%%%%%%%%%%%%%
\subsection{Comparative Performance}

\subsubsection{Dynamics Reconstruction}

Reconstruction accuracy is evaluated using the one-step-ahead \ac{NRMSE} over test trajectories:
$
\mathrm{NRMSE}(k)
=
\sqrt{\frac{1}{Nn_y}\sum_{k=1}^N\|\hat{\mathbf y}_k - \mathbf y_k\|_2^2},
$
where $\hat{\mathbf y}_k$ denotes the predicted state. \rev{Figure~\ref{fig:reconstructed_trajectories} shows the system trajectories reconstructed via the identified RC--Koopman dynamics} and Table~\ref{tab:nrmse_comparison} summarizes the results for each method on the two benchmark systems. For Duffing oscillator, \ac{HAVOK} achieves the highest accuracy, outperforming RC--Koopman and \ac{EDMD}. However, this comes at the cost of stability as shown in Fig.~\ref{fig:Eigenvalues}. \ac{HAVOK} yields 7 unstable eigenvalues, whereas RC--Koopman exhibits only 1 and \ac{EDMD} none. For differential-drive robot, all three methods yield comparable performance, yet \ac{EDMD} produces 4 unstable eigenvalues. These results demonstrate that RC--Koopman achieves a favorable balance between reconstruction accuracy and dynamical stability compared to \ac{EDMD} and \ac{HAVOK}. \rev{Note that since all compared methods utilize a least-squares formulation for operator identification, each is numerically optimal for its respective observable set. The satisfactory performance of RC--Koopman, therefore, stems from the reservoir's ability to generate an informative and numerically stable dictionary.}
% By encoding the history of the system into the lifted state $\boldsymbol{\psi}_k$, the resulting linear operator captures nonlinearities that memoryless DMD/EDMD approaches fail to resolve in the original state space.} \rev{While some methods may achieve high accuracy on one-step benchmarks (e.g., \ac{HAVOK} on the Duffing oscillator), the RC--Koopman framework prioritizes the spectral and numerical integrity of the identified operator. The stability and conditioning results shown in Table II suggest that RC--Koopman identifies a more robust representation of the underlying attractor, which is a necessary prerequisite for any subsequent control or long-term analysis.
\begin{figure}
    \centering
    \includegraphics[width=\linewidth]{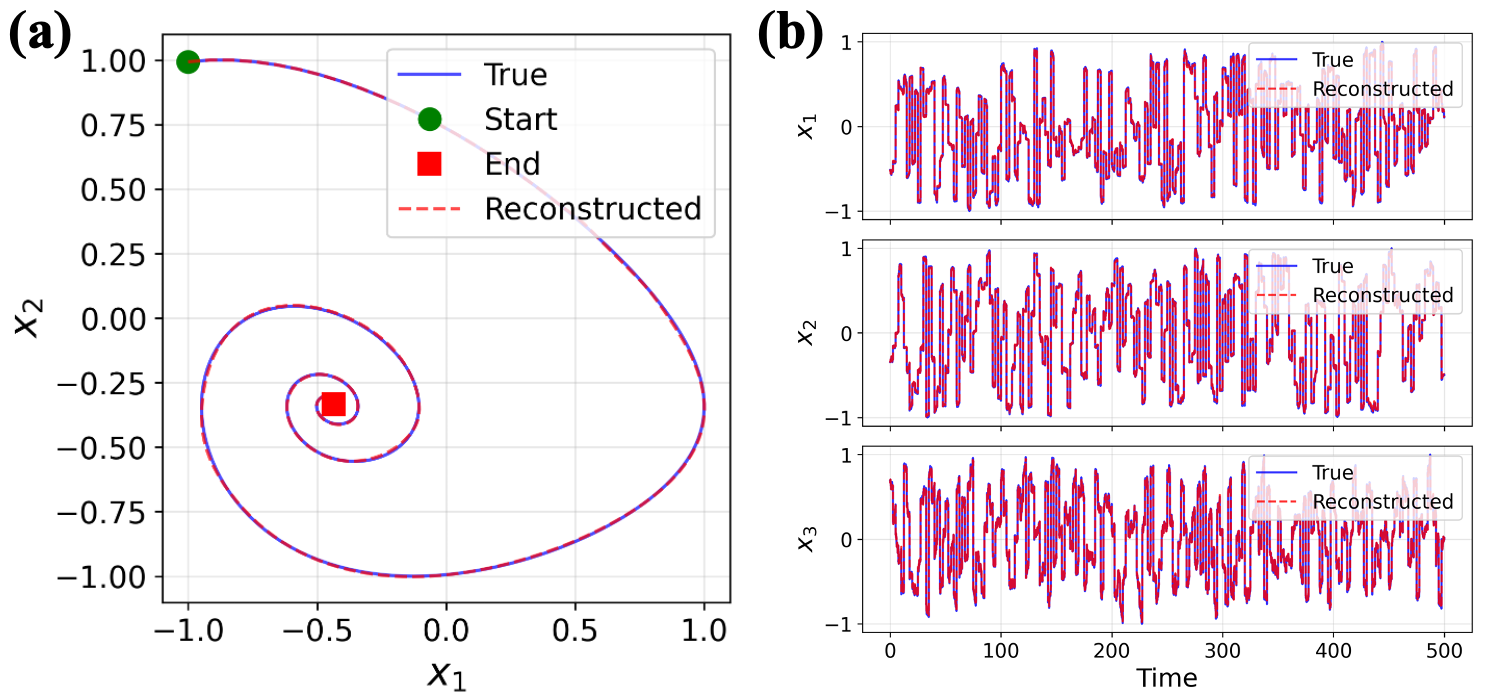}
    \caption{\rev{\textbf{Comparison of ground truth and reconstructed dynamics.} (a) The Duffing oscillator and (b) a differential-drive robot model.}}
    \label{fig:reconstructed_trajectories}
\end{figure}
\begin{table}
\centering
\caption{NRMSE comparison for reconstruction.}
\label{tab:nrmse_comparison}
\renewcommand{\arraystretch}{0.85} % Reduces row height
\setlength{\tabcolsep}{4pt}       % Reduces column spacing
\small                             % Slightly smaller font for compactness
\begin{tabular}{lcc}
\toprule
\multirow{2}{*}{Method} & \multicolumn{2}{c}{NRMSE} \\
\cmidrule(lr){2-3}
 & Duffing Oscillator & Differential-Drive Robot \\
\midrule
RC--Koopman & $5.38 \times 10^{-3}$ & $2.02 \times 10^{-1}$ \\
EDMD & $3.17 \times 10^{-2}$ & $2.02 \times 10^{-1}$ \\
HAVOK & $7.39 \times 10^{-4}$ & $2.03 \times 10^{-1}$ \\
\bottomrule
\end{tabular}
\end{table}

\begin{figure}
    \centering
    \includegraphics[width=\linewidth]{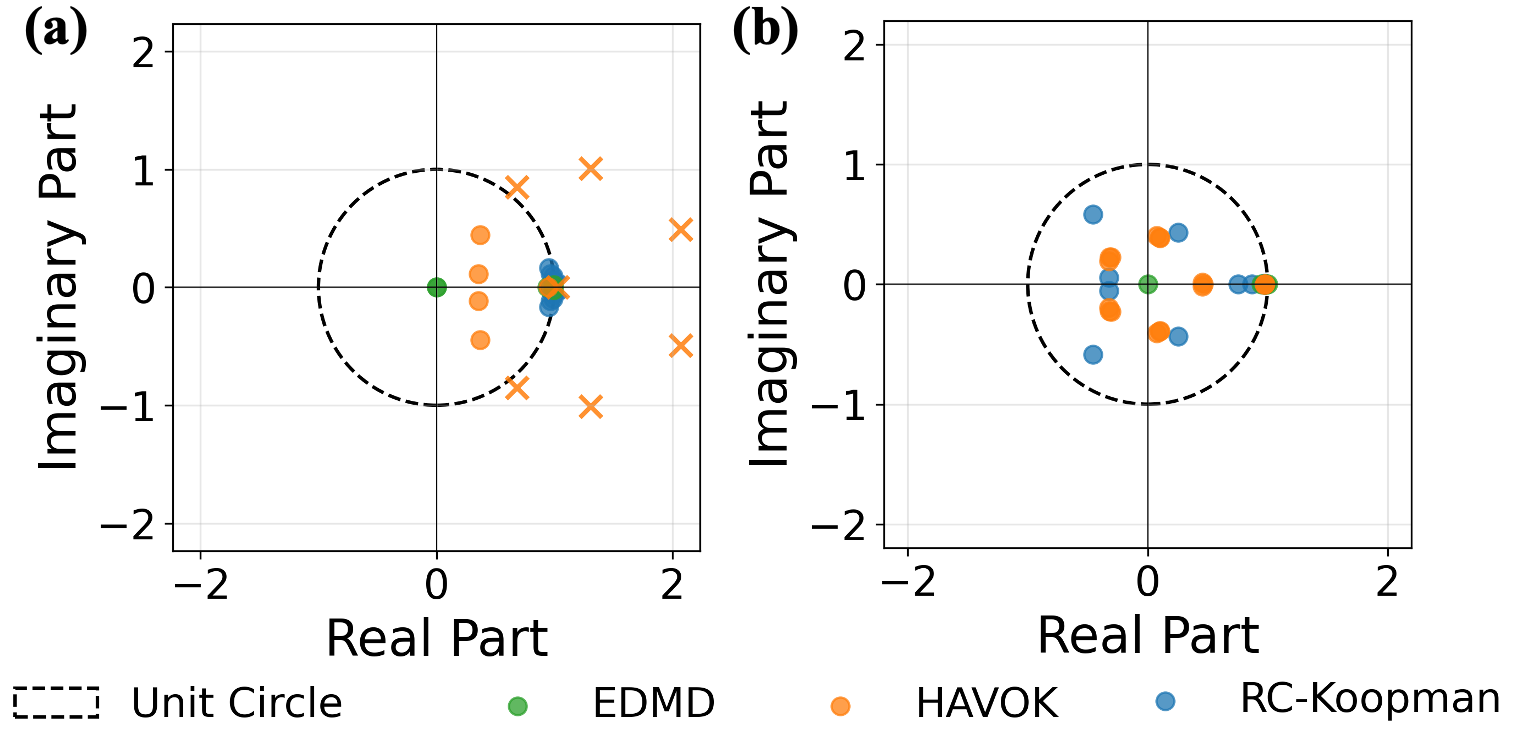}
    \caption{\textbf{Eigenvalue spectra of the learned Koopman operators.} (a) The Duffing oscillator and (b) a differential-drive robot. \rev{Dots ($\bullet$) indicate eigenvalues within or on the unit circle, while crosses ($\times$) denote eigenvalues outside the unit circle, representing numerically unstable modes.} For visual clarity, eigenvalues with $|\lambda| > 3$ are not shown.}
    \label{fig:Eigenvalues}
\end{figure}

\subsubsection{Numerical Conditioning}

The numerical conditioning of the lifted representations is summarized in Table~\ref{tab:conditioning_analysis}. For the Duffing oscillator, reservoir lifting satisfies the persistent excitation condition with $\alpha = 8.3\times 10^{-7}$, yielding a moderate condition number $\kappa(\mathbf{G})=1.6\times 10^{6}$ that adheres to the theoretical bound $C_\psi^2/\alpha$ in Proposition~\ref{prop:rc_conditioning_analysis}. Conversely, \ac{EDMD} and \ac{HAVOK} produce nearly singular Gram matrices ($\kappa \approx 10^{16}$--$10^{17}$) with minimum eigenvalues reaching numerical precision, effectively violating the persistent excitation condition. On the differential-drive robot, while \ac{EDMD} remains poorly conditioned, \ac{HAVOK} achieves $\kappa(\mathbf{G})=7.5\times 10^{2}$ and a positive excitation constant $\alpha = 2.7\times10^{-3}$. This difference reflects the sensitivity of Hankel embeddings to signal characteristics; delay coordinates may become rank-deficient for highly correlated or transient dynamics. Overall, the results demonstrate that reservoir lifting consistently produces well-conditioned feature representations. This empirical behavior supports Proposition~\ref{prop:rc_conditioning_analysis} and highlights the robustness of reservoir-based observables for Koopman operator identification.
\begin{table}[t]
\centering
\caption{Verification of Proposition~\ref{prop:rc_conditioning_analysis}.}
\label{tab:conditioning_analysis}
\renewcommand{\arraystretch}{0.85} % Reduces row height
\setlength{\tabcolsep}{4pt}       % Reduces column spacing
\small                             % Slightly smaller font for compactness
\begin{tabular}{lcccc}
\toprule
Method & $C_\psi$ & $\alpha$ & $\kappa(\mathbf G)$ & $C_\psi^2/\alpha$ \\
\midrule
\multicolumn{5}{c}{{Duffing Oscillator}} \\
RC--Koopman & $2.34$ & $8.3\times10^{-7}$ & $1.6\times10^{6}$ & $6.6\times10^{6}$ \\
EDMD & $2.23$ & $<10^{-16}$ & $1.2\times10^{17}$ & --- \\
HAVOK & $3.11$ & $<10^{-16}$ & $2.6\times10^{16}$ & --- \\
\midrule
\multicolumn{5}{c}{{Differential-Drive Robot}} \\
RC--Koopman & $2.53$ & $2.3\times10^{-4}$ & $3.8\times10^{3}$ & $2.8\times10^{4}$ \\
EDMD & $2.68$ & $<10^{-16}$ & $5.7\times10^{16}$ & --- \\
HAVOK & $3.67$ & $2.7\times10^{-3}$ & $7.5\times10^{2}$ & $5.0\times10^{3}$ \\
\bottomrule
\end{tabular}
\end{table}

\subsubsection{Koopman Spectral Observability}

To validate the spectral resolution predicted by Theorem~\ref{thm:koopman_resolution}, we analyze the relationship between identified Koopman eigenvalue lifetimes, $T_i = -1/\log |\lambda_i|$, and the reservoir memory horizon $\tau_\epsilon$. Scatter plots of $T_i$ versus $\tau_\epsilon$ illustrate the distribution of spectral components across varying reservoir spectral radii $\rho$, with the boundary $T_i = \tau_\epsilon$ denoting the theoretical limit of identifiable dynamics. The spectral distribution varies systematically with $\rho$ across both benchmark systems. Low spectral radii (e.g., $\rho=0.1$) restrict the memory horizon, capturing only a few dominant slow components. Conversely, excessive radii (e.g., $\rho \ge 1.5$) yield eigenvalues with lifetimes exceeding the effective memory horizon, suggesting temporal over-extension and unreliable identification. When $\rho$ is chosen via Algorithm~\ref{alg:spectral_radius_correlation} ($\rho \approx 0.98$), all identified eigenvalues lie within the $T_i \le \tau_\epsilon$ region. This alignment confirms that the reservoir memory effectively matches the system’s dominant time scales, enabling the RC--Koopman framework to capture the Koopman spectrum without introducing spurious, long-lived artifacts. These results empirically support the theoretical interpretation that $\rho$ governs the temporal resolution of the lifted feature space.

% Across both benchmark systems, the distribution of modes varies systematically with the spectral radius. For small spectral radii (e.g., $\rho = 0.1$ in the Duffing oscillator), the reservoir memory is short and only a small number of slow modes appear, with the largest lifetimes dominating the spectrum. Conversely, excessively large spectral radii (e.g., $\rho = 1.5$ or $2.0$ in the differential-drive robot) produce modes with large apparent lifetimes that exceed the effective memory horizon, indicating over-extended temporal dynamics that cannot be reliably resolved. For intermediate spectral radii, most modes cluster near the origin, corresponding to rapidly decaying components whose lifetimes remain short relative to the system time scales. Importantly, when the spectral radius is chosen based on Algorithm~\ref{alg:spectral_radius_correlation}, yielding $\rho \approx 0.98$ for both systems, all identified modes lie below the $T_i=\tau_\epsilon$ boundary. This indicates that the reservoir memory is well matched to the dominant dynamical time scales of the system and is sufficient to capture the observable Koopman modes without introducing unrealistically long-lived components. These observations are consistent with Theorem~\ref{thm:koopman_resolution}, which predicts that the reservoir memory horizon determines which Koopman modes remain observable in the lifted representation. The empirical results therefore support the theoretical interpretation that tuning the spectral radius effectively controls the temporal resolution of the Koopman feature space.

\begin{figure}
    \centering
    \includegraphics[width=\linewidth]{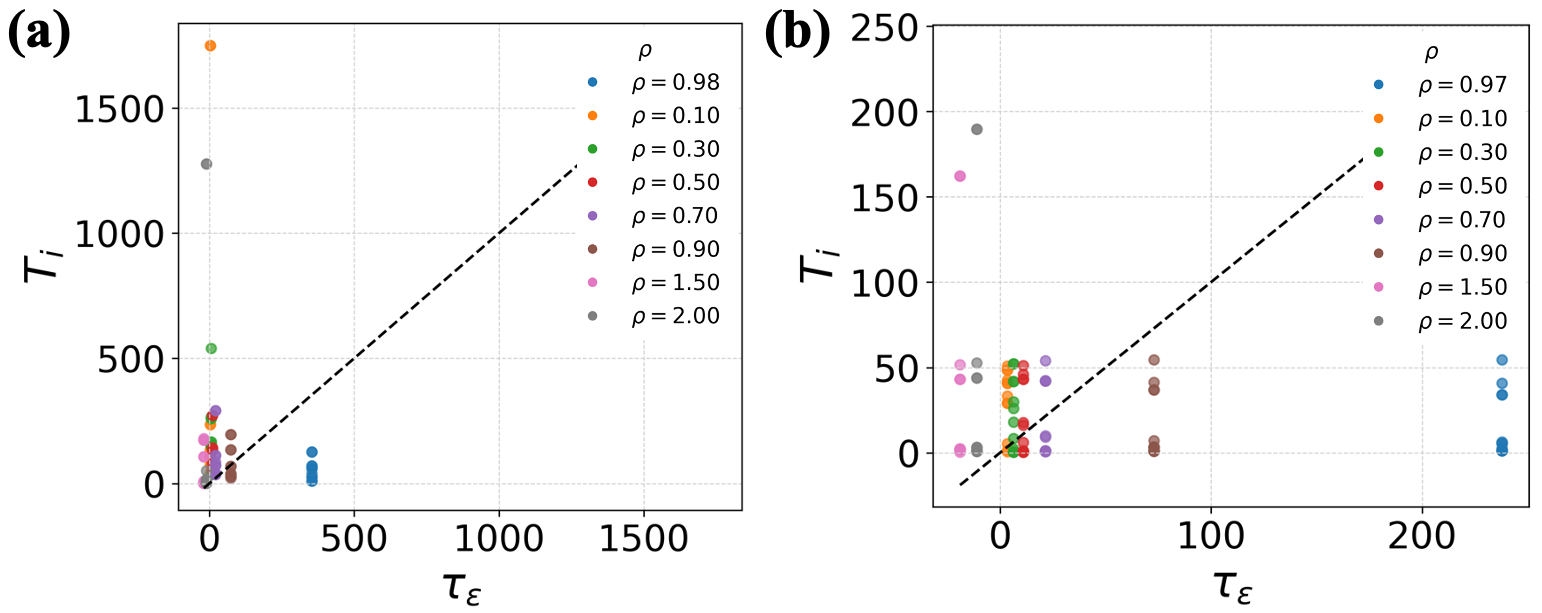}
    \caption{\textbf{Koopman eigenvalue lifetimes $T_i$ versus reservoir memory horizon $\tau_\epsilon$ across varying spectral radii $\rho$.} Each dot represents an identified eigenvalue, color-coded by $\rho$. The diagonal line $T_i=\tau_\epsilon$ indicates the boundary of identifiable dynamics: eigenvalues below the line evolve on time scales shorter than the reservoir memory and are therefore observable in the lifted representation. The spectral radius selected by the proposed correlation-based rule ($\rho\approx0.98$) places all spectral components within the observable region. (a) Duffing oscillator. (b) Differential-drive robot.}
    \label{fig:Observability}
\end{figure}

%%%%%%%
% \begin{itemize}
%     \item Choose a threshold $\eta \in (0,1)$ (e.g., $\eta = 0.9$).
%     \item For each spectral radius $\rho \in [0.1, 0.95]$:
%     \begin{itemize}
%         \item Construct the RC--Koopman model and compute the Koopman eigenvalues $\{\lambda_i(\rho)\}$.
%         \item Count the number of slow modes: $N_{\text{slow}}(\rho) = \#\{i : |\lambda_i(\rho)| \geq \eta\}$.
%     \end{itemize}
%     \item Plot $\rho$ versus $N_{\text{slow}}(\rho)$ to visualize how the spectral radius affects the number of resolvable slow modes.
% \end{itemize}

%%%%%%%
% For each eigenvalue $\lambda_i$, define the Koopman mode lifetime
% \[
% T_i = -\frac{1}{\log |\lambda_i|}.
% \]
% This represents the characteristic decay time of the mode.

% The reservoir memory horizon follows from Proposition~3:
% \[
% \tau_{\epsilon} = \frac{\log(\epsilon / \|W_{\text{in}}\|_2)}{\log \gamma},
% \]
% where $\gamma = L_{\sigma} \|W_{\text{res}}\|_2$.

% Plot $T_i$ versus $\tau_{\epsilon}$, or alternatively plot
% \[
% \max_i T_i \quad \text{vs} \quad \tau_{\epsilon}.
% \]

% Theorem prediction:
% \[
% T_i \lesssim \tau_{\epsilon}
% \]
% for observable modes.

%%%%%%%%%%%%%%%%%%%%%%%%%%%%%%%%%%%%%%%%%%%%%%%%%%%%

%%%%%%%%%%%%%%%%%%%%%%%%%%%%%%%%%%%%%%%%%%%%%%%%%%%%%%%%%%%%%%%%%%%%%%%%%%%%%%%%

\section{Conclusion}
\label{sec:conclusion}
This paper introduces RC--Koopman, a framework that utilizes reservoir as a stateful dictionary for linearizing nonlinear dynamical systems. By embedding fading-memory dynamics into the lifting process, the approach establishes a principled link between reservoir contraction and Koopman spectral observability. Theoretical and empirical results demonstrate that RC--Koopman achieves a superior balance between prediction accuracy and numerical conditioning compared to \ac{EDMD} and \ac{HAVOK}. \rev{Future work will investigate reservoir quality measures and extend the framework to long-horizon state estimation and trajectory reconstruction for robots with complex, nonlinear dynamics.}

\bibliography{reference}
\bibliographystyle{ieeetr}

\end{document}